\documentclass{article}

\usepackage[utf8]{inputenc}
\usepackage[T1]{fontenc}
\usepackage{hyperref}
\usepackage{url}
\usepackage{booktabs}
\usepackage{amsfonts}
\usepackage{amsmath}
\usepackage{nicefrac}
\usepackage{microtype}
\usepackage{amsthm}
\usepackage{xcolor}
\usepackage[numbers]{natbib}
\usepackage{graphicx}
\usepackage{longtable}
\usepackage{tikz}

\usepackage[margin=1.25in, top=1.25in, bottom=1.5in]{geometry}
\usepackage{mathpazo} 
\usepackage{setspace}
\usepackage{float}
\usepackage{minted}
\usepackage{parskip}
\usepackage{amsthm}

\usepackage{fontspec}
\usepackage{polyglossia}
\usepackage{newunicodechar}

\setdefaultlanguage{english}
\setotherlanguage{sinhala}
\setotherlanguage{malayalam}
\setotherlanguage{devanagari}
\setotherlanguage{any}

\newfontfamily{\sinhalafont}[Path=./, Script=Sinhala, Scale=1.05, AutoFakeBold=1.5]{NotoSinhala.ttf}
\newfontfamily{\malayalamfont}[Path=./, Script=Malayalam, Scale=1.05, AutoFakeBold=1.5]{NotoSansMalayalam.ttf}
\newfontfamily{\devanagarifont}[Path=./, Script=Devanagari, Scale=1.05, AutoFakeBold=1.5]{NotoSansDevanagari.ttf}
\newfontfamily{\dejavu}[Path=./, Scale=1.05, AutoFakeBold=1.5]{DejaVuSans.ttf}

\newtheorem{proposition}{Proposition}
\theoremstyle{definition}
\newtheorem{definition}{Definition}

\newcommand{\sn}[1]{{\sinhalafont #1}}
\newcommand{\dev}[1]{{\devanagarifont #1}}
\newcommand{\ml}[1]{{\malayalamfont #1}}
\newcommand{\dejv}[1]{{\dejavu #1}}

\title{Separate Before You Compress: The WWHO Tokenization Architecture}

\author{
  \textbf{Kusal Darshana} \\
  Remeinium Research \\
  \texttt{thekusaldarshana@remeinium.com}
}
\makeatletter
\def\thm@space@setup{%
  \thm@preskip=0.9em
  \thm@postskip=0.9em
}
\makeatother
\begin{document}

\maketitle

\begin{abstract}
Current Large Language Models (LLMs) mostly use BPE (Byte Pair Encoding) based tokenizers, which are very effective for simple structured Latin scripts such as English. However, standard BPE tokenizers struggle to process complex Abugida scripts due to their structural complexity. The problem is that these tokenizers break complex conjuncts, which are multi-codepoint grapheme clusters, into meaningless sub-character units. This degrades the LLM's reasoning efficiency by forcing it to learn basic orthographic structures at inference time and raises inference costs, resulting in a significant ``Token Tax" for the Global South. We propose a new three-layer architecture, the WWHO (Where-What-How Often), and an algorithm named SGPE (Syllable-aware Grapheme Pair Encoding) that separates the linguistic rules of the script from the statistical compression process while enabling seamless multilingual tokenization. 
Using Sinhala and Devanagari (Hindi/Sanskrit) as highly complex Abugida scripts, we trained WWHO on a cleaned 30-million-sentence dataset and evaluated on a 1,499,950-sentence test set. For Sinhala, SGPE achieves a Token to Word Ratio (TWR) of 1.274 with 4.83 characters per token, representing a 61.7\% reduction in tokens compared to OpenAI’s \texttt{o200k\_base}. For Hindi, it achieves a TWR of 1.181 (27.0\% reduction vs o200k). On the mixed-script (Sinhala, Devanagari, and English) dataset, SGPE achieves an overall TWR of 1.240, representing token reductions of 36.7\%, 39.6\%, and 60.2\% relative to o200k\_base, Llama 4 Scout, and DeepSeek V3, respectively. This effectively extends the usable context window by up to $4.38\times$ for these Abugida languages while ensuring a Linguistic Zero-Breakage Guarantee, which ensures that no valid syllable is ever split across multiple tokens.

\end{abstract}


\section{Introduction}
Tokenization plays a crucial role in defining a language model's performance \citet{ali2024tokenizer} and reasoning capabilities. This helps the model understand the language; if it fails to capture the true morphological structure of the language, its reasoning and the quality of its responses in that language are also poor. When a tokenizer breaks a language's atomic grapheme clusters, it leads to three major issues beyond the side effects. First, it breaks the linguistic meaning of the cluster, blocks the model’s linguistic understanding, then increases the fertility (the average number of tokens a tokenizer generates per word) and reduces the usable context window. \citet{goldman2024unpacking} show that tokenizers’ compression efficiency correlates with downstream performance. Increasing the number of tokens decreases the LLM's reasoning capacity by distributing attention across more tokens. Third, since the Transformer's self-attention mechanism scales quadratically $O(N^2)$, a larger number of tokens increases inference cost. This affects the economy of both users and companies by imposing a ``Token Tax’’ \citep{ahia2023all, petrov2023language}. Over a billion people globally use Abugida\footnote{\url{https://www.unicode.org/events/utw/2023/talks/abugida/}}, pay more, but get less. 

Byte Pair Encoding (BPE) \citep{gage1994bpe, sennrich2016neural} is effective for English because its atomic unit is the byte. However, complex Abugida scripts are structured into grapheme clusters composed of elements such as conjuncts, diacritics (pili), virama (\texttt{U+0DCA}, \sn{්}), and Zero-Width Joiner (\texttt{U+200D}), etc., all of which utilize Unicode. Unicode characters frequently require two or more bytes, leading to BPE fragmentation in these scripts. For example, OpenAI's o200k\_base tokenizer splits the Sinhala word ``\sn{ආයුබෝවන්}" into [`\dejv{�}', `\dejv{�}', \sn{'ය', 'ු', 'බ', 'ෝ', 'ව', 'න්'}]. Such splits disrupt both the orthographic integrity and semantic coherence of the language. As a result, Sinhala and other Abugida scripts often exceed a 3.0 token-to-word ratio (TWR), whereas English typically achieves a TWR near 1.3. This discrepancy results in a reduction of up to $4.38\times$ in the context window for these languages (section 5). Additionally, this fragmentation can lead to hallucinations due to LLMs' limited linguistic understanding. It's often observed that when LLMs generate text, they sometimes mix it with other scripts, and we hypothesize that this is caused by poor tokenization alongside poor data quality. Here are some examples: ``\dev{	पा}\sn{සල}'', ``\ml{പരി}\sn{ගණකය}''.

In this work, using Sinhala and Devanagari (Hindi/Sanskrit) as highly complex Abugida scripts to demonstrate the architecture, we introduce \textbf{WWHO (Where-What-How Often)}, a three-layer architecture that prioritizes the linguistic integrity before applying any statistical compression.  Layer 1, the Router, identifies script boundaries to route segments between Latin and script-specific processing units with $O(N)$ time complexity. Layer 2, named LinguisTrie, utilizes a table-driven Deterministic Finite Automaton (DFA) to group raw Unicode characters into complete, atomic syllables (defined in section 3). This layer guarantees that complex conjuncts or clusters are never broken into meaningless parts. We present a formal definition using regular expressions that represents the entire linguistic algorithm for how any syllable form is built from elements such as consonants, vowels, viramas (HAL), conjuncts, etc. (section 3). As layer 3, we extend the Grapheme Pair Encoding (GPE) concept by  \citet{velayuthan2025gpe} into Syllable-aware GPE (SGPE), which applies the statistical merging to the pre-defined syllables. WWHO further combines SGPE with standard BPE through a Unified Meta-Vocabulary to support code-switching without ID collisions. In this way, WWHO helps LLMs better understand the language by providing well-formed syllables and pass-through characters with a formal zero-breakage guarantee, freeing the language model from the burden of learning basic character reconstruction, restoring both linguistic integrity and computational efficiency at the source.

\section{Background}
The world is working to address this BPE fragmentation in Abugida scripts, as it has become a significant issue affecting over a billion people. Among recent works, \citet {velayuthan2025gpe} proposed Grapheme Pair Encoding (GPE), which operates on Unicode grapheme clusters to avoid byte-level fragmentation. We extend the GPE into SGPE (Syllable-aware GPE) by adding a syllabification layer called LinguisTrie, which utilizes a DFA (Deterministic Finite Automaton) to ensure perfect syllable extraction before statistical compression. Additionally, \citet{punctuation2019tok} developed an enhanced, punctuation-based Sinhala tokenizer. Recent efforts to improve Sinhala LLMs, such as SinLlama \citep{aravinda2025sinllama}, focused on enhancing the Llama-3-8B tokenizer with Sinhala-specific vocabulary.

Historically, in the Sinhala NLP, \citet{weerasinghe2005rule} have established a rule-based syllabification algorithm for Sinhala. A broader record of such developments is documented in the survey by \citet{nisansa2019survey}, which includes the Sinling library\citep{sinling}, a toolkit for various Sinhala NLP tasks, including tokenization. However, these traditional methods are often hardcoded for a specific script. WWHO focuses on a script-independent approach, allowing any Abugida or complex script type to be dynamically injected, without changing the code. 

In addition to these works, \citet{clark2022canine} presented a tokenization-free encoding concept, focusing on efficiency and tokenization‑free modeling. A study by \citet{dagan2024getting} explores several pre-tokenization approaches based on the BPE architecture. However, it was mainly focused on English, code, and general multilingual tokenization.
Among related works, \citet{myanmar2008fsa} have proposed a rule-based syllabification approach that utilizes FSA (Finite State Automaton) to demonstrate the syllable
structure of Burmese (Myanmar) script. \citet{htay2008myanmar} have presented a Word Segmentation using syllable-level longest matching for the Burmese language.
 Although these efforts often mainly focus on ensuring the linguistic accuracy within a single language, they lack the real-world multilingual and code-switching (cross-script) tokenization required by frontier LLMs. WWHO fulfills this gap by providing a unified framework that can be integrated with existing systems, without deprecating or replacing them.

\section{A Regular-Language Framework for Abugida Syllables}
Most Abugida scripts share a common linguistic structural invariant: consonants possess an inherent vowel, which is either suppressed by a virama (HAL) or modified by diacritics. Complex conjuncts are formed through the interaction of viramas and Zero-Width Joiners (ZWJ). We hypothesize that syllable segmentation for this family can be formalized as a regular language. 

\begin{definition}
A perfect, atomic orthographic syllable in a linguistic script is further indivisible, and if divided, one or more segments will become individually meaningless.
\end{definition}

\subsection{Character Classes}
We define the following distinct abstract classes based on orthographic roles for Sinhala and Devanagari(Hindi/Sanskrit).
The regular expressions are defined based on the logical Unicode storage order (e.g., a base consonant followed by its dependent vowel signs) rather than the visual rendering sequence.

\begin{table}[H]
\caption{Sinhala Unicode Character Class}
\label{tab:char_classes}
\begin{center}
\resizebox{\columnwidth}{!}{%
\begin{tabular}{lll}
\toprule
Class & Unicode Range & Description \\
\midrule
\textbf{C} (consonant) & U+0D9A--U+0DC6 & Base consonants (vyanjana) \\
\textbf{V} (vowel) & U+0D85--U+0D96 & Independent vowels (svara) \\
\textbf{P} (pili) & U+0DCF--U+0DDF, U+0DF2--U+0DF3 & Dependent vowel signs (diacritics) \\
\textbf{H} (HAL) & U+0DCA & Virama (vowel suppressor / ``Hal-Lakuna'') \\
\textbf{Z} (ZWJ) & U+200D & Zero-Width Joiner (conjunct connector) \\
\textbf{M} (post-modifier) & U+0D82, U+0D83 & Anusvara and visarga (nasalization and aspiration markers) \\
\textbf{O} (other) & All else & Non-Sinhala passthrough \\
\bottomrule
\end{tabular}}
\end{center}
\end{table}

\subsubsection{Formal Definition for Sinhala:}
A valid Sinhala syllable can be defined over the alphabet $\Sigma_s = \textbf{C} \cup \textbf{V} \cup \textbf{P} \cup \textbf{H} \cup \textbf{Z} \cup \textbf{M}$ matching the regular expression
\begin{equation}
S_s = C (H Z? C)^* (P \mid H)? M? \;\;\mid\;\; V M? \,.
\end{equation}
This definition captures all linguistically valid syllable forms that consist of independent vowels, conjuncts, viramas (HAL), ZWJs, diacritics, and optional post-modifiers. 

\begin{table}[H]
\caption{Devanagari Unicode Character Class}
\label{tab:devanagari_char_classes}
\begin{center}
\resizebox{\columnwidth}{!}{%
\begin{tabular}{@{}cll@{}}
\toprule
\textbf{Class} & \textbf{Unicode Range} & \textbf{Description} \\
\midrule
\textbf{C} (consonant) & U+0915--U+0939, U+0958--U+095F, U+0978--U+097F & Base consonants (including nukta-precomposed) \\
\textbf{V} (vowel) & U+0904--U+0914, U+0960--U+0961, U+0972--U+0977 & Independent vowels \\
\textbf{P} (matra) & U+093E--U+094C, U+094E--U+094F, U+0955--U+0957, U+0962--U+0963 & Dependent vowel signs (matras) \\
\textbf{H} (halant) & U+094D & Virama (vowel suppressor / conjunct connector) \\
\textbf{Z} (ZWJ/ZWNJ) & U+200D, U+200C & Zero-Width Joiner and Non-Joiner \\
\textbf{N} (nukta) & U+093C & Nukta (diacritic dot for borrowed phonemes) \\
\textbf{M} (modifier) & U+0900--U+0903, U+093D, U+0950--U+0954, U+1CD0+, U+A8E0+ & Anusvara, chandrabindu, visarga, avagraha, Vedic accents \\
\textbf{O} (other) & All else & Non-Devanagari passthrough \\
\bottomrule
\end{tabular}}
\end{center}
\end{table}

\paragraph{Formal Definition for Devanagari:}
A valid Devanagari syllable is defined over the alphabet $\Sigma_d = \textbf{C} \cup \textbf{V} \cup \textbf{P} \cup \textbf{H} \cup \textbf{Z} \cup \textbf{N} \cup \textbf{M}$ matching the regular expression:
\begin{equation}
    S_d = C N? (H Z? C N?)^* (P \mid H)? M^* \mid V M^*
\end{equation}
The Devanagari grammar also incorporates the Nukta (N) diacritic. Unlike Sinhala, Devanagari, especially Vedic texts, use multiple modifiers; this definition allows for having $M^*$.

\paragraph{Pass-through tokens:}
 Any character $c \notin \Sigma$ is considered as under the set $O$ (Other) for both alphabets and is treated as a pass-through token to handle noisy or invalid Unicode text.

\subsection{The Valid Language Schema}
We define the requirements for a valid language schema to ensure the structural integrity of the syllabification process.
\begin{definition}
\label{def:valid_schema}
A language schema is considered \textbf{valid} if it satisfies the following conditions.
    \begin{enumerate}
        \item \textbf{Disjoint Character Classes:} All defined character classes within the schema must be disjoint. No Unicode codepoint should be assigned to more than one named class simultaneously.
        \item \textbf{Completeness:} The schema must declare at least one non-empty Unicode block range. Every codepoint assigned to a named class must fall within that declared block. Codepoints within the block that are not assigned to any named class are implicitly assigned to class $O$, ensuring no codepoint in the script block is left unclassified.
        \item \textbf{Determinism:} The DFA transition function is deterministic, which means that each state has at most one transition for a given class. Undefined transitions (marked as `---') are explicit boundary signals used to trigger the emission of the last accepted syllable.
        \item \textbf{Emit State Isolation:} Both emit states, ORPHAN and PASSTHROUGH, must have no outgoing transitions, and must only be reachable directly from the START state. This ensures that the emit states produce only one single string, before the scanner continues to the START state to extract the next syllable.
        \item \textbf{Grammar Alignment:} The language recognized by the DFA via its set of accept states must be identical to the language generated by the formal grammar $S$ defined within the schema (e.g., $S_s$ or $S_d$).
        \item \textbf{Maximal-Munch Consistency:} The DFA must operate as a Maximal Munch (greedy longest-match) scanner. Instead of stopping at the first valid syllable boundary, the machine must continue to evaluate subsequent characters and extract the longest possible valid syllable.

        e.g: When considering the string `\sn{ක්‍රෝ}', it is made up as: \sn{ක + ් + [ZWJ] + ර + ෝ}. Instead of stopping at the syllable \sn{`ක'} and extracting it, the DFA must continue to grab the full syllable \sn{`ක්‍රෝ'}.
    \end{enumerate}
\end{definition}

\subsection{The Language is Regular}
\begin{proposition}
Most Abugida scripts, including Sinhala and Devanagari, can be represented by regular expressions, and therefore, they are regular languages.
\end{proposition}

\textbf{Proof.} Because $S$ ($S_s$ or $S_d$) is a regular expression and regular expressions define regular languages. 

We present the following DFA (Deterministic Finite Automaton) transition tables, declared according to the `Valid Language Schema' conditions as defined, for Sinhala and Devanagari.

\begin{table}[H]
\centering
\caption{DFA Transition Table for Sinhala Syllable Recognition}
\label{tab:sinhala_dfa}
\resizebox{\columnwidth}{!}{%
\begin{tabular}{@{}cccccccc@{}}
\toprule
\textbf{State} & \textbf{C} & \textbf{V} & \textbf{H} & \textbf{P} & \textbf{Z} & \textbf{M} & \textbf{O} \\ 
\midrule
START        & IN\_CLUSTER & IN\_VOWEL & ORPHAN     & ORPHAN     & ORPHAN    & ORPHAN & PASSTHROUGH \\
IN\_CLUSTER  & $-$         & $-$       & HAL\_SEEN  & PILI\_SEEN & $-$       & ACCEPT & $-$ \\
HAL\_SEEN    & IN\_CLUSTER & $-$       & $-$        & $-$        & ZWJ\_SEEN & ACCEPT & $-$ \\
ZWJ\_SEEN    & IN\_CLUSTER & $-$       & $-$        & $-$        & $-$       & $-$    & $-$ \\
PILI\_SEEN   & $-$         & $-$       & $-$        & ORPHAN     & $-$       & ACCEPT & $-$ \\
IN\_VOWEL    & $-$         & $-$       & $-$        & $-$        & $-$       & ACCEPT & $-$ \\
ACCEPT       & $-$         & $-$       & $-$        & $-$        & $-$       & $-$    & $-$ \\
\bottomrule
\end{tabular}%
}

\vspace{2mm}
\begin{center}
\small
\textbf{Legend:} C = Consonant, V = Independent Vowel, H = Hal/Virama, P = Dependent Vowel (Pili), \\
Z = Zero-Width Joiner, M = Modifier (Anusvara/Visarga), O = Other. \\
$-$ = No valid transition (syllable boundary signal) 
\end{center}
\end{table}

\begin{table}[H]
\centering
\caption{DFA Transition Table for Devanagari Syllable Recognition}
\label{tab:devanagari_dfa}
\resizebox{\columnwidth}{!}{%
\begin{tabular}{@{}ccccccccc@{}}
\toprule
\textbf{State} & \textbf{C} & \textbf{V} & \textbf{H} & \textbf{P} & \textbf{Z} & \textbf{N} & \textbf{M} & \textbf{O} \\ 
\midrule
START         & IN\_CLUSTER & IN\_VOWEL & ORPHAN     & ORPHAN     & ORPHAN    & ORPHAN & ORPHAN       & PASSTHROUGH \\
IN\_CLUSTER   & $-$         & $-$       & HAL\_SEEN  & PILI\_SEEN & $-$       & NUKTA\_SEEN & IN\_CLUSTER\_M & $-$ \\
NUKTA\_SEEN   & $-$         & $-$       & HAL\_SEEN  & PILI\_SEEN & $-$       & $-$    & IN\_CLUSTER\_M & $-$ \\
HAL\_SEEN     & IN\_CLUSTER & $-$       & $-$        & $-$        & ZWJ\_SEEN & $-$    & ACCEPT       & $-$ \\
ZWJ\_SEEN     & IN\_CLUSTER & $-$       & $-$        & $-$        & $-$       & $-$    & $-$          & $-$ \\
PILI\_SEEN    & $-$         & $-$       & $-$        & ORPHAN     & $-$       & $-$    & IN\_CLUSTER\_M & $-$ \\
IN\_CLUSTER\_M & $-$        & $-$       & $-$        & $-$        & $-$       & $-$    & IN\_CLUSTER\_M & $-$ \\
IN\_VOWEL     & $-$         & $-$       & $-$        & $-$        & $-$       & $-$    & IN\_VOWEL\_M & $-$ \\
IN\_VOWEL\_M  & $-$         & $-$       & $-$        & $-$        & $-$       & $-$    & IN\_VOWEL\_M & $-$ \\
ACCEPT        & $-$         & $-$       & $-$        & $-$        & $-$       & $-$    & $-$          & $-$ \\
\bottomrule
\end{tabular}%
}

\vspace{2mm}
\begin{center}
\small
\textbf{Legend:} C = Consonant, V = Independent Vowel, H = Virama/Halant, P = Dependent Vowel (Matra), \\
Z = Zero-Width Joiner/Non-Joiner, N = Nukta, M = Modifier (Anusvara/Chandrabindu/Visarga etc.), O = Other. \\
$-$ = No valid transition (syllable boundary signal) 
\end{center}
\end{table}

The DFA identifies a syllable boundary using the states. The accepting states for Sinhala are IN\_CLUSTER, IN\_VOWEL, HAL\_SEEN, ZWJ\_SEEN, PILI\_SEEN, and ACCEPT. For Devanagari, the set is extended to include NUKTA\_SEEN, IN\_CLUSTER\_M, and IN\_VOWEL\_M.

\subsection{Zero-Breakage Guarantee}
\begin{definition}A tokenizer $T$ satisfies the Zero-Breakage Guarantee for a given language schema with grammar $S$ if it satisfies two conditions.
Let $T(w) = t_1,t_2,t_3,...,t_n$ be the sequence of tokens produced for an input string $w$.
\begin{enumerate}
    \item Losslessness: The concatenation of all tokens must perfectly reconstruct the original string $w$ ($w = t_1\cdot t_2 \cdot t_3 ... \cdot t_n$), with the exception that any character missing from the vocabulary or removed during pruning is replaced with the \texttt{[UNK]} token.
    \item Linguistic Integrity:   Each token $t_i$ is composed of one or more complete atomic syllables accepted by $S$, or a single passthrough character (class $O$ or $ORPHAN$).
\end{enumerate}
\end{definition}

\section{The WWHO Architecture}
The WWHO breaks down the tokenization process into first principles using three layers, based on the concepts of Where, What, and How Often. This approach enables more effective training of multilingual tokenizers by addressing Abugida issues and seamlessly integrating with current frontier tokenization approaches, without replacing existing methods. The architecture is completely language independent. The linguistic constraints and formal grammar of a specific script, such as the character class, the DFA transition table, and Unicode block ranges, are separated from the core algorithm and defined dynamically via an external \texttt{JSON} Schema, allowing WWHO to support any Abugida or complex script without modifying the codebase, simply by swapping the schema file. To add a new linguistic script, it's enough to add only a new language-specific \texttt{JSON} file.  Both Router and LinguisTrie use the same schema files.

\begin{figure}[h]
  \centering
  \includegraphics[scale=0.07]{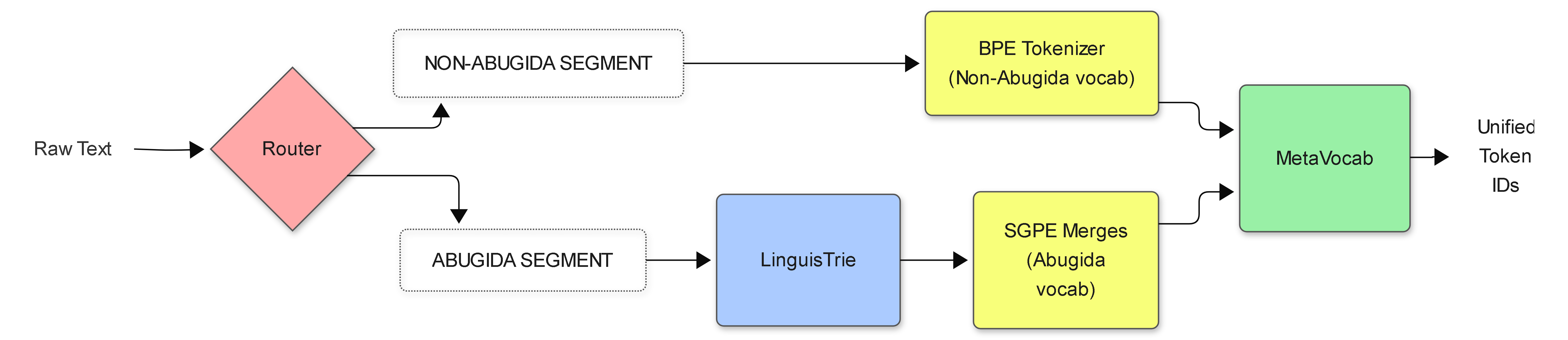}
  \caption{WWHO architecture.}
  \label{fig:arch}
\end{figure}

\subsection{Layer 1: Router (The ``Where'')}
The real-world data are multilingual code-switching text, rather than pure monoscripted text. Since BPE is highly effective for Latin scripts such as English, we don't deprecate it. Instead, the WWHO implements a targeted, localized routing layer that scans the text in a single pass using a Unicode range block scanner. The scanned text is cleanly partitioned into Abugida and Non-Abugida segments. In this partitioning, the router detects leading spaces and flags them for the next LinguisTrie layer to attach to. Once segmented, the text is routed to its respective specialists. 

\subsubsection{Unicode-based Script Identification}
To demonstrate the architecture, we use 3 different script types as follows.
\begin{enumerate}
    \item Sinhala: \texttt{[U+0D80]-[U+0DFF]}
    \item Devanagari: \texttt{[U+0900]-[U+097F]}
    \item Other: which do not belong to the above two categories, such as Latin/ASCII/Code/Digits/Emojis, etc.  
\end{enumerate}

Symbols such as Devanagari Punctuation (Dandas - \texttt{(U+0964)}, \texttt{(U+0965)}) are also considered as native characters and are included in the relevant block.

\subsubsection{Script Boundaries}
The router performs an O(N) linear scan over the text, classifying each character in O(1) time to identify script transitions. It enforces a Hard Script Boundary even in the absence of whitespace, resulting in an explicit partitioning of text into distinct language segments.
\begin{center}
e.g., `\sn{ඇ}pple' is segmented into `\sn{ඇ}' and `pple'
\end{center}

\paragraph{ZWJ/ZWNJ Handling:} 
Since complex conjuncts in such Abugida scripts are made up with invisible Unicode characters such as Zero-Width-Joiner (U+200D) and Zero-Width-Non-Joiner (U+200C), the router's identification logic extends beyond simple Unicode range checks to explicitly detect these control characters, preventing incorrect script segmentation. The router explicitly absorbs such joiners into the correct segment, preventing incorrect script boundaries. This further assists the next layer, LinguisTrie, in performing accurate syllabification, and it's fundamental to ensuring Zero Breakage Guarantee.

\subsubsection{Routing Policy}
Each segment is then routed to the appropriate tokenizer. Then the SGPE encodes Abugida segments, while other Non-Abugida segments are encoded by the standard BPE. This decoupling enables SGPE integration with existing Latin-specialized BPE tokenizers while preserving the linguistic integrity of complex Abugida scripts.

\subsection{Layer 2: LinguisTrie (The ``What'')}
The LinguisTrie, a DFA-based lexical scanner, is the layer responsible for grouping the Abugida text segments sent by the router, according to their Unicode codepoints, into Atomic Orthographic Syllables following the DFA and grammar defined in Section 3.  It operates with $O(N)$ time complexity by performing a single linear scan over the text.

\subsubsection{The Syllabification Algorithm}
\paragraph{Leading space handling:}
LinguisTrie captures the leading space flagged by the router and attaches it to the first identified syllable in a string. It attaches only one space character, and other whitespace (tabs, newlines) are emitted standalone.

\begin{proposition}
The LinguisTrie correctly extracts Atomic Orthographic Syllables from any valid Unicode string defined via a valid language schema as defined in Section 3.2 (Definition~\ref{def:valid_schema}).
\end{proposition}

\begin{proof}
LinguisTrie strictly follows the mechanism below to extract perfect, atomic orthographic syllables from a given string:
\begin{enumerate}
    \item \textbf{Direct Grammar Mapping:} Every transition of the DFA is designed to follow the syllable grammar, the Regex ($S_s$ or $S_d$) defined within the script. Hence, the machine only moves forward if a syllable can be extracted from the given string.
    \item \textbf{Maximal Munch (Greedy Longest Match):} The machine doesn't stop immediately once it finds a syllable. Instead, it continues to check whether the syllable can have a further expanded form. This strategy ensures that complex conjuncts joined by viramas (e.g., \sn{`න්ද්‍ර`}, \sn{`ල්ප`}) are captured as a single syllable unit. This is a desirable outcome to reduce the overall token count.
    \item \textbf{{last\_accept\_pos:}} The machine remembers the last accept state at which it found a valid transition according to the grammar. When moving forward, it never fails if it meets an invalid character. It moves back and emits the last valid syllable. Hence, the syllable is always a valid atomic orthographic syllable.
    \item \textbf{Passthrough and Orphan handling:} If it finds a character belonging to class $O$ (Other), which is outside of the script's alphabet or if it finds an $ORPHAN$, which can be a part of a syllable, but individually meaningless (eg: an individual HAL sign), it emits that as a single string to ensure the Losslessness principle.
\end{enumerate}
\end{proof}

\textbf{Example Sinhala:}

\begin{tabular}{@{}l@{ }l@{}}
\textbf{Input} & : \sn{චන්ද්‍රයාගේ ආලෝකය පෘථිවියට ක්‍ෂණයෙන්/ක්ෂණිකව නොලැබේ.} \\[6pt]
\textbf{Syllables} & : [\sn{'ච', 'න්ද්\u200dර', 'යා', 'ගේ', ' ආ', 'ලෝ', 'ක', 'ය', ' පෘ', 'ථි', 'වි', 'ය', 'ට',}\\
& \phantom{: [}\sn{' ක්\u200dෂ', 'ණ', 'යෙ', 'න්', '/', 'ක්ෂ', 'ණි', 'ක', 'ව', ' නො', 'ලැ', 'බේ', '.']} \\[6pt]
\end{tabular}

\textbf{Example Devanagari Syllabification:}

\begin{tabular}{@{}l@{ }l@{}}
\textbf{Input} & : \dev{विद्यालय में पढ़ाई होती है।} \\[6pt]
\textbf{Syllables} & : [\dev{'वि'}, \dev{'द्या'}, \dev{'ल'}, \dev{'य'}, \dev{' में'}, \dev{' प'}, \\
& \phantom{: [}\dev{'ढ़ा'}, \dev{'ई'}, \dev{' हो'}, \dev{'ती'}, \dev{' है'}, \dev{'।'}] \\[6pt]
\end{tabular}

\subsection{Layer 3: SGPE and Meta Vocabulary}
\subsubsection{Syllable-aware Grapheme Pair Encoding}
While SGPE only merges the well-extracted, word boundary-aware atomic Abugida syllable pairs, which are also Unicode grapheme clusters, all non-Abugida text segments are processed by an existing BPE algorithm, and, for the demonstration, we used OpenAI's o200k\_base as the BPE tokenizer.
Velayuthan et al. (2025) introduced the GPE (Grapheme Pair Encoding) concept by replacing BPE's atomic unit (the byte) with Unicode grapheme clusters. SGPE extends this principle one level further: the atomic unit is a perfect Atomic Orthographic Syllable as defined in Section 3. 
Other than the following differences, SGPE is the same BPE algorithm:
\begin{enumerate}
\item \textbf{Syllabic initialization:} The base Abugida vocabulary is initialized with atomic syllables and passthrough characters emitted by LinguisTrie.
\item \textbf{Boundary-aware scoping:}
Unlike BPE, SGPE merges consider the word boundaries by any non-Abugida characters, punctuations, digits, in addition to white spaces.
\item \textbf{Pruning:} 
To ensure the vocabulary allows for more unique and frequently used merges, SGPE implements a pruning mechanism that replaces syllables below a defined threshold, the prune frequency, as `\texttt{[UNK]}' tokens. The result is that any Unicode sequence that appears less than the prune frequency is not included in the vocabulary, keeping it cleaned from most malformed Unicode sequences or typographical noise and valid but extremely rare tokens.
\end{enumerate}

\begin{proposition}
The WWHO architecture ensures the Zero-Breakage Guarantee.
\end{proposition}

\begin{proof}
Since the Router performs only a simple partitioning operation on a given string $w$ without any deletion or modification, the character sequence is preserved. Then, as a complete scanner, LinguisTrie ensures no character is discarded at the syllabification stage.  Since the SGPE layer only performs the merging on these units without modifying the characters, the Losslessness requirement ($w = t_1\cdot t_2 \cdot t_3 ... \cdot t_n$) is satisfied (subject to \texttt{[UNK]}).

Inheriting Proof of Proposition 2 made in Section 4.2, SGPE only merges either perfect Atomic Orthographic Syllables accepted by $S$ or fallback strings (Orphan/Passthrough). This satisfies the 2nd requirement for the Zero-Breakage guarantee, and the proposition follows.
\end{proof}

\subsection{Unified Meta-Vocabulary}
We introduce the Unified Meta-Vocabulary approach to prevent Abugida tokens processed by SGPE and non-Abugida tokens processed by BPE IDs from colliding with each other.

Let $V_{BPE}$ be the vocabulary size of the BPE tokenizer (e.g., $V_{BPE}=200,019$ for \texttt{o200k\_base}), and $V_{SGPE}$ be the size of the tokenized SGPE vocabulary. The WWHO architecture sequentially maps the ID spaces into a contiguous array:
\begin{align*}
    \text{ID } 0 \dots (V_{BPE}-1) &\longrightarrow \text{non-Abugida (via } \texttt{BPE}) \\
    \text{ID } V_{BPE} \dots (V_{BPE} + V_{SGPE} - 1) &\longrightarrow \text{Abugida (via SGPE)}
\end{align*}

During tokenization, a $V_{BPE}$ offset is applied to all SGPE-generated IDs to align them within the unified meta-vocabulary. This prevents ID collisions and enables the model to deal with multiple scripts in a single embedding space. In inference, the process is reversed for detokenization.

\subsubsection{End-to-End Execution Trace}
To illustrate the complete workflow, consider the code-switched input: ``\sn{ඔයා} 1 special \dev{अद्भुत}''.

\begin{enumerate}
    \item \textbf{Router (Layer 1):} The input is partitioned into script-specific segments: the Sinhala segment ``\sn{ඔයා}'', the Other segment `` 1 special'', and the Devanagari segment `` \dev{अद्भुत}''. The leading space before the Devanagari word is flagged for internal attachment.
    \item \textbf{LinguisTrie (Layer 2):} The Abugida segments are decomposed into atomic orthographic syllables. The Sinhala segment produces \texttt{["\sn{ඔ}", "\sn{යා}"]} and the Devanagari segment produces \texttt{["\dev{ अ}", "\dev{द्भु}", "\dev{त}"]}.
    \item \textbf{SGPE and BPE (Layer 3):} Final tokens are generated based on the learned vocabulary. Non-Abugida segments are processed via the foundation BPE, while Abugida syllables are merged by SGPE.
\end{enumerate}

The final tokenization result is:
\begin{center}
\texttt{["\sn{ඔයා}", " ", "1", " special", "\dev{ अद्भुत}"]}
\end{center}

\section{Training \& Evaluation}
In this section, we evaluate the performance of the WWHO architecture across three key scripts (Sinhala, Devanagari, and English) and compare it against frontier tokenizers: OpenAI’s \texttt{o200k\_base}, Meta’s \texttt{Llama 4 Scout}, and \texttt{DeepSeek V3}. We report the Token-to-Word Ratio (TWR), the characters per token (CPT), and the relative token reduction.

For training and evaluation, we created a 30-million sentence corpus \citep{wwho30m} aggregating four open-source datasets to represent a balanced mix of Sinhala (35\%), Hindi(45\%), and English(20\%). The final corpus consists of: 
\begin{enumerate}
\item MADLAD\_CulturaX\_cleaned \citep{polyglots}
\item A filtered subset of CC100-Sinhala \citep{cc100sinhala} containing only Sinhala and English code-mixed sentences.
\item Two subsets of MADLAD-400 (pure Hindi and pure English) \citep{allen2023madlad400} 
\item A filtered subset of HINMIX\_hi-en \citep{hicm2024synthetic} containing only Hindi and English code-mixed sentences.
\end{enumerate}

The corpus was split 95\%/5\% into training and test sets, yielding a 1,499,950-sentence evaluation partition containing 122 million characters.

\subsection{Training Configurations}
We set the training hyperparameters as follows, and the evaluation results are based on these configurations.

\begin{table}[H]
\caption{SGPE Training Hyperparameters}
\label{tab:hyperparams}
\begin{center}
\begin{tabular}{lc}
\toprule
Parameter & Value \\
\midrule
Unified Meta-Vocabulary Size & 328,020 \\
BPE (o200k\_base) Vocab Size & 200,019 \\
SGPE Vocab Size & 128,000$^\dagger$ \\
Prune frequency threshold $\theta$ & 100 \\
Training corpus size & $\sim$28.5 M sentences \\
Evaluation corpus size & $\sim$1.5 M sentences \\
Code-switching routing & Enabled \\
\bottomrule
\end{tabular}
\end{center}
\end{table}

\paragraph{$^\dagger$Note:} The SGPE vocabulary consists of 128,000 tokens, as the hyperparameter is set to 128,000. Since the LinguisTrie attaches space as a prefix to the next syllable, the standalone space token is not included in the SGPE vocabulary. Therefore, one space token is injected later, as an architectural design to ensure the system's integrity. Hence, the total SGPE vocabulary size is 128,001 tokens.

\subsection{Tokenization Efficiency (TWR, CPT \& Reduction)}
The Token-to-Word Ratio (TWR) measures how many tokens are required to represent a single word on average. A TWR closer to 1.00 indicates that the tokenizer is effectively treating each word as a single unit without any fragmentation.

\begin{table}[H]
\centering
\caption{Tokenization Efficiency Comparison Across Scripts}
\label{tab:tokenization_comparison}
\resizebox{\columnwidth}{!}{%
\begin{tabular}{@{}llrrrr@{}}
\toprule
\textbf{Script} & \textbf{Tokenizer} & \textbf{Total Tokens} & \textbf{TWR} & \textbf{CPT} & \textbf{\% Reduction (vs SGPE)} \\
\midrule
Sinhala & SGPE (WWHO)    & 6,654,288  & 1.274 & 4.83 & ---    \\
        & OpenAI (o200k) & 17,360,196 & 3.324 & 1.85 & 61.7\% \\
        & Llama 4 Scout  & 18,157,707 & 3.476 & 1.77 & 63.4\% \\
        & DeepSeek V3    & 29,152,698 & 5.581 & 1.10 & 77.2\% \\
\midrule
Hindi   & SGPE (WWHO)    & 13,433,554 & 1.181 & 4.29 & ---    \\
        & OpenAI (o200k) & 18,394,075 & 1.617 & 3.13 & 27.0\% \\
        & Llama 4 Scout  & 19,566,121 & 1.720 & 2.94 & 31.3\% \\
        & DeepSeek V3    & 31,682,218 & 2.786 & 1.82 & 57.6\% \\
\midrule
English & SGPE (WWHO)    & 7,240,147  & 1.330 & 4.46 & ---    \\
        & OpenAI (o200k) & 7,420,527  & 1.364 & 4.35 & 2.4\%  \\
        & Llama 4 Scout  & 7,512,843  & 1.381 & 4.30 & 3.6\%  \\
        & DeepSeek V3    & 7,904,670  & 1.453 & 4.09 & 8.4\%  \\
\midrule
Overall & SGPE (WWHO)    & 27,327,989 & 1.240 & 4.47 & ---    \\
        & OpenAI (o200k) & 43,174,798 & 1.959 & 2.83 & 36.7\% \\
        & Llama 4 Scout  & 45,236,671 & 2.053 & 2.70 & 39.6\% \\
        & DeepSeek V3    & 68,739,586 & 3.119 & 1.78 & 60.2\% \\
\bottomrule
\end{tabular}%
}
\end{table}

\begin{table}[H]
\centering
\caption{Example Tokenizations Across Tokenizers}
\label{tab:example_tokenizations}
\resizebox{\columnwidth}{!}{%
\begin{tabular}{@{}llll@{}}
\toprule
\textbf{Word} & \textbf{Tokenizer} & \textbf{Tokens} & \textbf{Count} \\
\midrule
\sn{ව්යාකරණය}
  & SGPE (Ours)    & \sn{['ව්යා', 'කරණය']}                                                              & 2 \\
  & OpenAI (o200k) & \sn{['ව්', 'යා', 'ක', 'රණ', 'ය']}                                                  & 5 \\
  & Llama 4 Scout  & \sn{['ව්', 'යා', 'කර', 'ණය']}                                                      & 4 \\
  & DeepSeek V3    & \sn{['ව', ''්', 'ය', 'ා', 'ක', 'ර',} \dejv{'�'}, \dejv{'�'}, \sn{'ය']}            & 9 \\
\midrule
\sn{ශ්‍රී ලංකාව}
  & SGPE (Ours)    & \sn{['ශ්\textbackslash u200dරී', ' ලංකාව']}                                                      & 2 \\
  & OpenAI (o200k) & \sn{['ශ්', '\textbackslash u200dරී', ' ලංක', 'ාව']}                                              & 4 \\
  & Llama 4 Scout  & \sn{['ශ්', '\textbackslash u200dර', ''ී', ' ල', 'ං', 'ක', 'ාව']}                                  & 7 \\
  & DeepSeek V3    & [\dejv{'�'}, \dejv{'�'}, `\sn{්}', '\textbackslash u200d', \sn{'ර', ' 'ී'}, \dejv{'�'}, \dejv{'�'}, \dejv{'�'}, \dejv{'�'}, \sn{'ක', 'ා', 'ව']} & 13 \\
\midrule
\dev{अंतर्राष्ट्रीय}
  & SGPE (Ours)    & \dev{['अंतर्राष्ट्रीय']}                                                            & 1 \\
  & OpenAI (o200k) & \dev{['अ', 'ंतर', '्र', 'ाष्ट्रीय']}                                                & 4 \\
  & Llama 4 Scout  & \dev{['अ', 'ंतर', '्र', 'ाष्ट्रीय']}                                                & 4 \\
  & DeepSeek V3    & \dev{['अ', 'ंत', 'र', '्र', 'ाष', '्ट', '्री', 'य']}                                & 8 \\
\midrule
\dev{कृत्रिम बुद्धिमत्ता}
  & SGPE (Ours)    & \dev{['कृत्रिम', ' बुद्धिमत्ता']}                                                   & 2 \\
  & OpenAI (o200k) & \dev{['क', 'ृ', 'त्र', 'िम', ' बुद्ध', 'िम', 'त्ता']}                               & 7 \\
  & Llama 4 Scout  & \dev{['क', 'ृ', 'त्र', 'िम', ' ब', 'ुद्ध', 'िम', 'त्ता']}                           & 8 \\
  & DeepSeek V3    & \dev{['क', 'ृ', 'त्र', 'िम', ' ब', 'ुद', '्ध', 'िम', 'त्त', 'ा']}                  & 10 \\
\bottomrule
\end{tabular}%
}
\end{table}

\subsection{Key Observations}
\paragraph{The ``Token Tax" Reduction:} For Sinhala text, SGPE demonstrates a massive 61.7\%-77.2\% reduction in token count, allowing 4.83 characters per token. While frontier models often fragment a single Sinhala word into 3-10 tokens with meaningless sub-word pieces, SGPE maintains a TWR of 1.274 while ensuring the Zero-Breakage Guarantee. And for Devanagari (Hindi), SGPE achieves 1.181 TWR and 4.29 CPT, reducing 27.0\%-57.6\% token reduction.

\paragraph{Overall Efficiency:} Across a mixed multilingual test set containing Sinhala, Hindi, and English text in a 35\%:45\%:20\% ratio, SGPE reduces the total token count by 36.7\%-60.2\% compared to the industry standard frontier tokenizers.

\paragraph{Reduction on English?:}
An unexpected observation during evaluation was the 2.4\%-8.4\% reduction in token count for the English compared to other tokenizers. While WWHO routes Latin text directly to the foundation BPE tokenizer, the observed efficiency gain is a result of unintentional code-switching within real-world datasets.
\textbf{The Router's Contribution:} In datasets classified as English, there can be a minor percentage of Abugida texts(e.g., proper nouns, locations, or localized terms). While the foundation BPE tokenizer fragments these foreign segments into byte-level tokens, the WWHO's Router accurately isolates these Abugida clusters. Then LinguisTrie and SGPE perfectly tokenize them. This localized efficiency results in the observed token reduction even in English-dominant contexts. 

To verify this observation and that the routing layer does not modify English segments while processing, we conducted a pure ASCII Stress Test using 10,000 sentences containing zero Abugida characters. In this controlled environment, the token counts produced by WWHO and o200k\_base were 100\% identical (0.0\% difference), which means all the ASCII text was processed only by the BPE tokenizer.

\subsection{Context Window Capacity Analysis}
One of the most impactful factors of the efficiency of a tokenizer is the LLM context window. The context window capacity multiplier, which measures the context window saving by a tokenizer compared to another, can be calculated by the following formula.
\begin{equation} \text{Capacity Multiplier} = \frac{\text{$Tokens_{baseline}$}}{\text{$Tokens_{sgpe}$}}
\end{equation}
\begin{table}[H]
\centering
\caption{Effective Context Window Capacity  Multiplier Comparison}
\label{tab:context_window_savings}
\resizebox{\columnwidth}{!}{%
\begin{tabular}{@{}llrrl@{}}
\toprule
\textbf{Script} & \textbf{Tokenizer} & \textbf{SGPE Tokens} & \textbf{Other Tokens} & \textbf{Context Multiplier} \\
\midrule
Sinhala & vs OpenAI (o200k) & 6,654,288  & 17,360,196 & 2.61$\times$ \\
        & vs Llama 4 Scout  & 6,654,288  & 18,157,707 & 2.73$\times$ \\
        & vs DeepSeek V3    & 6,654,288  & 29,152,698 & 4.38$\times$ \\
\midrule
Hindi   & vs OpenAI (o200k) & 13,433,554 & 18,394,075 & 1.37$\times$ \\
        & vs Llama 4 Scout  & 13,433,554 & 19,566,121 & 1.46$\times$ \\
        & vs DeepSeek V3    & 13,433,554 & 31,682,218 & 2.36$\times$ \\
\midrule
English & vs OpenAI (o200k) & 7,240,147  & 7,420,527  & 1.02$\times$ \\
        & vs Llama 4 Scout  & 7,240,147  & 7,512,843  & 1.04$\times$ \\
        & vs DeepSeek V3    & 7,240,147  & 7,904,670  & 1.09$\times$ \\
\midrule
Overall & vs OpenAI (o200k) & 27,327,989 & 43,174,798 & 1.58$\times$ \\
        & vs Llama 4 Scout  & 27,327,989 & 45,236,671 & 1.66$\times$ \\
        & vs DeepSeek V3    & 27,327,989 & 68,739,586 & 2.52$\times$ \\
\bottomrule
\end{tabular}%
}
\end{table}

As shown in our evaluation, SGPE provides a very effective context window capacity expansion for Abugida scripts. This allows over a billion Abugida users to use LLMs more effectively and get better responses. \textbf{The Token Tax} will have a significant drop via the WWHO architecture.

\subsection{Vocabulary Quality \& Glitch Token Prevention}
Standard BPE tokenizers often allocate vocabulary capacity to statistical anomalies, leading to ``glitch tokens" or orphaned Unicode joiners (e.g., isolated ZWJ or Viramas). In Abugida scripts, the generation of such artifacts by an LLM causes low response quality. To evaluate vocabulary quality, we conducted a Glitch Token detection test across the 1.5-million sentence evaluation corpus.
The results confirm that SGPE produces 0 glitch tokens for both Sinhala and Devanagari scripts. Out of 128,001 SGPE tokens (Sinhala and Devanagari), only 1,394 (roughly 1\%) remained unused during evaluation. This rate shows that by enforcing syllable-level linguistic accuracy, the SGPE algorithm avoids creating meaningless fragments and effectively uses its vocabulary for real-world multilingual script structures.

\subsection{Empirical Validation of Zero-Breakage Guarantee}
To verify the Zero-Breakage Guarantee defined and formally proven in Sections 3 and 4, we conducted a Round-Trip test using the same test corpus. Across 122 million mixed multilingual characters, SGPE achieved 100\% round-trip consistency with zero non-[UNK] mismatches and a negligible [UNK]-caused loss rate of 0.08\%, practically verifying the theoretical proof.

\section{Conclusion}
In this work, we presented the WWHO, a novel DFA-based multilingual tokenization architecture, and SGPE, the merging algorithm used in WWHO.

We proposed that the basic, atomic unit of a linguistic script should be the Syllable, extending the GPE concept to SGPE. By separating linguistic rules from statistical compression, we ensure that complex grapheme clusters remain intact throughout the process.

WWHO satisfies the Zero-Breakage Guarantee both theoretically and practically, by ensuring that meaningful syllables are never fragmented. Our evaluation shows that SGPE achieves up to a 77.2\% reduction in token count for Sinhala and 57.6\% for Hindi compared to frontier tokenizers. This efficiency allows 4.29-4.83 characters per token and expands the effective context window by up to $4.38\times$. These results prove that preserving linguistic integrity directly reduces the ``Token Tax" for Abugida languages.

We plan to scale the WWHO framework to support Abugida and other complex scripts. Future versions of this work will include extensive downstream evaluations and empirical comparisons with existing tokenizers to further validate the performance of the WWHO architecture. The full source code, evaluation harness, and developer notes are open-sourced on GitHub \citep{wwho2026code}, and the full dataset we used to train and evaluate is available at Hugging Face \citep{wwho30m}. Pre-trained artifacts, including the \texttt{vocab.json}, \texttt{tokenizer.json}, and optimized merge rules, are available via Hugging Face \citep{wwho2026tokenizer}.

\section*{Acknowledgements}
The author is deeply grateful to Dr. Nisansa de Silva for all the invaluable support and guidance on this research. We thank the \citet{polyglots}, Allen AI \citep{allen2023madlad400}, \citet{hicm2024synthetic}, Mr. \citet{cc100sinhala} for open-sourcing the datasets we used to create ours \citep{wwho30m}.
We also acknowledge \citet{tiktoken}, \citet{llama4scout}, and \citet{deepseekv3} for their contributions by open-sourcing frontier tokenizers, which served as essential baselines for our evaluation.

\bibliographystyle{plainnat} 
\bibliography{refs}

\end{document}